%% file: main.tex
\documentclass{article}

\usepackage[english]{babel}

\usepackage[a4paper,top=2cm,bottom=2cm,left=3cm,right=3cm,marginparwidth=1.75cm]{geometry}

\usepackage[round]{natbib}

\usepackage{url}
\usepackage{booktabs}

\input{preamble}

\title{Generative Modeling under Non-Monotone MAR Missingness via Approximate Wasserstein Gradient Flows}

\author{%
    Gitte Kremling \\
    \normalsize Department of Mathematics\\
    \normalsize University of Hamburg\\
    \normalsize \texttt{gitte.kremling@uni-hamburg.de}
    \and
    Jeffrey N\"af \\
    \normalsize Research Institute for Statistics\\\normalsize and Information Science\\
    \normalsize University of Geneva\\
    \normalsize \texttt{jeffrey.naf@unige.ch}
    \and
    Johannes Lederer \\
    \normalsize Department of Mathematics\\
    \normalsize University of Hamburg\\
    \normalsize \texttt{johannes.lederer@uni-hamburg.de}
}

\date{}

\begin{document}

\maketitle

\input{content/abstract}

\input{content/intro}
The code to replicate all experiments can be found on \url{https://github.com/gkremling/FLOWGEM}.

\input{content/background}

\input{content/method}

\input{content/results}

\input{content/conclusion}


\bibliographystyle{abbrvnat}
\bibliography{literature}

\appendix
\input{content/appendix}

\end{document}

%% file: preamble.tex
\usepackage{amsmath}
\usepackage{amssymb}
\usepackage{amsthm}
\usepackage{physics}
\usepackage{mathtools}
\usepackage{tikz} 
\usepackage[svgnames]{xcolor}
\usepackage[ruled, noline, linesnumbered]{algorithm2e}
\usepackage{xspace}
\DontPrintSemicolon

\newcommand{\R}{\mathbb{R}}
\renewcommand{\P}{\mathbb{P}}
\newcommand{\expec}[2][]{\mathbb{E}_{#1}\left[#2\right]}
\newcommand{\KL}[1]{\ensuremath{\mathrm{KL}}\left(#1\right)}
\newcommand{\NA}{\text{NA}}
\DeclareMathOperator*{\argmax}{arg\,max}
\DeclareMathOperator*{\argmin}{arg\,min}
\newcommand{\de}{\, \mathrm d }

\newcommand{\eps}{\varepsilon}
\newcommand{\Gauss}[2]{\mathcal{N}\left(#1,#2\right)}
\newcommand{\bigo}[1]{\mathcal{O}\qty(#1)}

\newcommand{\PMmXx}[1]{\mathbb{P}(M=#1 \mid X=x)}
\newcommand{\PMmXxm}[1]{\mathbb{P}(M=#1 \mid X^{(m)}=x^{(m)})}

\newcommand{\pX}{\pi}
\newcommand{\pXnew}{\rho}
\newcommand{\dfunc}{g}
\newcommand{\name}{FLOWGEM\xspace}

\newtheorem{lemma}{Lemma}

\newtheorem{proposition}[lemma]{Proposition}

%% file: content/abstract.tex
\begin{abstract}
    The prevalence of missing values in data science poses a substantial risk to any further analyses. Despite a wealth of research, principled nonparametric methods to deal with general non-monotone missingness are still scarce. Instead, ad-hoc imputation methods are often used, for which it remains unclear whether the correct distribution can be recovered. In this paper, we propose \name, a principled iterative method for generating a complete dataset from a dataset with values Missing at Random (MAR). Motivated by convergence results of the ignoring maximum likelihood estimator, our approach minimizes the expected Kullback–Leibler (KL) divergence between the observed data distribution and the distribution of the generated sample over different missingness patterns. To minimize the KL divergence, we employ a discretized particle evolution of the corresponding Wasserstein Gradient Flow, where the velocity field is approximated using a local linear estimator of the density ratio. This construction yields a data generation scheme that iteratively transports an initial particle ensemble toward the target distribution.
    Simulation studies and real-data benchmarks demonstrate that \name achieves state-of-the-art performance across a range of settings, including the challenging case of non-monotone MAR mechanisms. 
    Together, these results position \name as a principled and practical alternative to existing imputation methods, and a decisive step towards closing the gap between theoretical rigor and empirical performance.

\end{abstract}

%% file: content/intro.tex
\section{Introduction}

Increasing volumes of data have made missing values an almost ubiquitous challenge across disciplines. Despite the steep upturn of research on the topic, there is still a substantial gap between theory and practice. This is particularly true in the challenging case of general non-monotone missingness illustrated in Figure \ref{fig:illustrationfocond}: Instead of observing independently and identically distributed data directly, observations are masked according to an unknown missingness mechanism. While this missingness can be monotone, such as in the left example in Figure \ref{fig:illustrationfocond}, complex non-monotone patterns of missingness are possible.

  \begin{figure*}[!ht]
    \centering
\begin{tikzpicture}


\node at (0,0) {$ \begin{pmatrix}
x_{1,1} & x_{1,2} & x_{1,3} \\
\textrm{NA} & x_{2,2} & x_{2,3} \\
\textrm{NA} & \textrm{NA} & x_{3,3}
\end{pmatrix}$};

\node at (4,0) {$\begin{pmatrix}
x_{1,1} & x_{1,2} & x_{1,3} \\
x_{2,1} & \textrm{NA} & x_{2,3}\\
 \textrm{NA} & x_{3,2} & x_{3,3} 
\end{pmatrix}$};


\node at (8,0) {$\begin{pmatrix}
x_{1,1} & x_{1,2} & x_{1,3} \\
 x_{1,2} &\textrm{NA} & x_{2,3} \\
\textrm{NA} & x_{3,2} &\textrm{NA} 
\end{pmatrix}$};

\end{tikzpicture}
    \caption{Three data matrices with missing values, each with three different patterns. The matrix on the left shows monotone missingness, while the middle and right show cases with non-monotone missingness. The middle matrix is the one used in the Example of Section \ref{sec:emp_res}.}
    \label{fig:illustrationfocond}
\end{figure*}

Following Rubin's classical categorization \citep{Rubin1976, whatismeant3}, the missing value mechanism can be divided into three categories of increasing complexity: Missing Completely at Random (MCAR), Missing at Random (MAR) and Missing Not at Random (MNAR). In this categorization, MCAR is generally understood to be too weak, while MNAR does not allow for identification in general \citep{ourresult}. As such MAR is of special interest and has taken a prominent role in the applied and methodological literature on missing data. At the same time, MAR is a target of frequent misunderstandings, as evidenced by the number of papers simply discussing the condition \citep{whatismeant, whatismeant3, näf2024goodimputationmarmissingness}. This, in addition to the scattered literature, might explain why there has been limited theoretical development for the case of non-monotone MAR. In particular, while \citet{Rubin1976} introduced MAR to justify a parametric maximum likelihood approach that allows to model only the data distribution and not the missingness mechanism (thus ``ignoring'' the missingness), consistency and asymptotic normality of this approach was only shown as recently as \citet{Takai2013}. Inspecting their approach and also later papers such as \citet{Mestimatormissingvalues} reveals that these consistency results in the complex case of non-monotone general MAR are deeply connected with the minimization of the Kullback-Leibler (KL) divergence between observed and proposed distributions. In particular, consistency does not hold for general M-estimation under MAR. Using this connection in a Bayesian framework, \citet{Bayes} derived a minimax rate estimation result for nonparametric density estimation under non-monotone MAR missingness. Based on this, they developed a generative method that takes a dataset contaminated with missing values and returns a sample from the uncontaminated distribution. The generated sample then allows to calculate consistent estimates of any continuous functional of the density. However, as the method is based on Gaussian mixtures, its performance on complex real data is less than ideal, as the authors note themselves.

On the other hand, there is a wealth of empirically well-performing imputation methods that tackle (non-monotone) MAR missingness. Leaving the observed variables unchanged, these methods try to ``fill in'' the missing values. This has led to a belief that the goal of imputation is to find the best prediction of the missing values. However, as argued in \citet{VANBUUREN2018, näf2024goodimputationmarmissingness} and others, the goal is really to obtain a generative method that recreates the original uncontaminated distribution, to be able to achieve unbiased estimates of functionals of interest in a second step. As such, both the generative method of \citet{Bayes} and the wealth of imputation methods arguably share the same goal of recreating the uncontaminated distribution. While there is interesting theory on parametric imputation \citep{RobinsImputationTheory, MICE_Results, MICE_Results2, Guan2024}, practitioners often resort to nonparametric imputation such as versions of MICE \citep{mice} and modern neural net and machine learning methods such as MIWAE \citep{MIWAE}, GAIN \citep{GAIN}, or MIRI \citep{MIRI}. However, these nonparametric imputation methods lack theoretical grounding. In particular, it is not even clear whether they identify the right distribution under MAR in a population setting, that is, under perfect estimation, although important progress was made recently in \citep{MIRI} as outlined in Section~\ref{sec:more_literature}.

We attempt to close the gap between the imputation methods that perform empirically well but lack a theoretical underpinning and the generative method of \citet{Bayes} that is theoretically sound but empirically underperforming. For this, we follow the idea of the ignoring maximum likelihood estimator.
Our main contributions are as follows:
\begin{itemize}
    \item We show that, under MAR and a positivity condition, the true distribution emerges as the minimizer of the KL divergence on observed values averaged over all patterns. 
    \item We adapt the approximate Wasserstein gradient flow approach of \citet{liu2024minimizing} to develop a nonparametric sampling method, which we call Flow-based Generation for Missing data (\name), that generates a new sample from the uncontaminated distribution by minimizing the average KL divergence between the new distribution and the observed data.
    \item We demonstrate that \name is both theoretically sensible---the gradient approximation error vanishes asymptotically---and empirically competitive, matching and exceeding the performance of state-of-the-art methods. 
\end{itemize}
The new algorithm simply uses a local linear approximation and is thus fast to fit, while the gradient flow approach allows for great flexibility. All of this is possible with a minimal set of tuning parameters that can be chosen with heuristics that tend to work well for a wide range of datasets. In fact, our empirical section shows that our method matches or exceeds state-of-the-art methods on data ranging from simulated data with $d=3$ to real data with widely varying variances and $d=272$, all using the same setting.

Just as in \citet{misgan, missdiff, Bayes}, our approach is motivated by the fact that data analysts are typically interested in aspects of the data distribution rather than the data themselves. Consequently, both a completed version of the original dataset, as targeted in imputation methods, as well as a newly generated complete dataset following the same distribution, as targeted in generative methods, serve the purpose of the subsequent analysis. Additionally, new samples do not expose the original data points, making it very useful for privacy-sensitive applications. 

The paper proceeds as follows: We first dive into the background of the maximum likelihood estimator under MAR and an approximate Wasserstein gradient flow for complete data, including a further literature comparison, in Section~\ref{sec:Background}. In Section \ref{sec:method}, we derive our methodology, \name, adapting the Wasserstein gradient flow framework to the missing data setting and establishing consistency of the velocity field approximation. Section \ref{sec:emp_res} shows the impressive empirical performance of \name in a challenging simulated example, designed to test imputation methods for their ability to deal with MAR, as well as a broad range of real datasets. Finally, in Section~\ref{sec:concl}, we summarize our findings and outline directions for future research. Proofs and further technical details are provided in the Appendix.

%% file: content/background.tex
\section{Background}\label{sec:Background}

This section establishes background on missing values, the ignoring maximum likelihood estimator, as well as the Wasserstein gradient flow approximation of \citet{liu2024minimizing} with complete data. 

We start by introducing the notation used throughout the paper. Let $X \in \R^d$ denote a random vector with distribution $\P_X$ and density $\pX$, representing a complete observation of all feature variables, and $M \in \{0,1\}^d$ denote a missingness pattern with (discrete) distribution $\P_M$ and support $\mathcal{M} = \{m \in \{0,1\}^d ~|~ \P(M=m) > 0\}$. 
Given an independently and identically distributed sample $(X_i, M_i)_{i=1}^n$ of the joint distribution $\P_{X,M}$, in practice, we only observe the masked observations
\begin{equation*}
    X_{i,j}^*\,=\, 
    \begin{cases}
        X_{i,j}&~~~\text{for}~M_{i,j}=0\\
        \NA&~~~\text{for}~M_{i,j}=1
    \end{cases},
    \quad \text{with } i \in \{1,\dots,n\},~ j \in \{1, \dots, d\}.
\end{equation*}
For example, if the complete observation $X_1 = (X_{1,1}, X_{1,2}, X_{1,3})$ is associated to the missingness pattern $M_1 = (0, 1, 0)$, then we observe $X_1^*=(X_{1,1}, \NA, X_{1,3})$. For $m \in \{0,1\}^d$, let $X^{(m)} = (X_j \mid m_j=0, ~j\in \{1,\dots,d\})$ denote the sub-vector of $X$ corresponding to the observed values according to $m$. 
In the previous example, we have $X_1^{(M_1)} = (X_{1,1}, X_{1,3})$. The distributions of $X^{(m)}$ and $X^{(m)} \mid M=m$ are denoted by $\P_X^{(m)}$ and $\P_{X|M=m}^{(m)}$ and their densities by $\pX^{(m)}$ and $\pX_m^{(m)}$, respectively. The goal of this paper is to generate a new sample $\{\tilde{X_i}\}_{i=1}^{\tilde{n}}$ of fully observed data points that (approximately) follow the true distribution~$\P_X$. 

Different assumptions can be made about the dependence between $X$ and $M$. The missingness mechanism is called Missing Completely at Random (MCAR) if $\P_{X,M} = \P_X \otimes \P_M$, that is, if $X$ and $M$ are independent, and, on the other extreme, Missing Not at Random (MNAR) if the dependence structure between $M$ and $X$ can be arbitrary. If, on the other hand, it holds that $\PMmXx{m} =\PMmXxm{m}$, the mechanism is called Missing at Random (MAR). We note that, here, $m$ occurs on both sides of the probability statement, making the condition more complex than it may first seem. In particular, MAR does not mean that $M$ is conditionally independent of $X$ given $X^{(m)}$. We refer to the discussion in \citet{whatismeant3, näf2024goodimputationmarmissingness}.

\subsection{The ignoring maximum likelihood estimator}
\label{sec:mle}


Under missing values, we essentially observe draws from $\P_{X|M=m}^{(m)}$ for different patterns $m$. MAR does not prohibit distribution shifts between patterns and can lead to arbitrary complex non-monotone patterns. This is one of the key reasons why results under general MAR are difficult to obtain.

Nonetheless, if a parametric model $\P_\theta$ is used to approximate the distribution $\P_X$, the parameter $\theta$ can be estimated via maximum likelihood under MAR. As shown in \citet{Golden2019} and discussed in \citet[Theorem~2.2]{cherief2025parametric}, the maximum likelihood estimator (MLE)  $\theta_n^{\text{ML}}$ \emph{ignoring the missing value mechanism} converges $\P_{X,M}$-a.s. to
\begin{equation}
\label{eq:ML_limit}
    \theta_\infty^{\text{ML}} = \argmin_{\theta \in \Theta} \expec[M \sim \P_M]{\KL{\P_{X|M}^{(M)} \,|\!|\, \P_\theta^{(M)}}}
\end{equation}
as the sample size $n$ goes to infinity, assuming that appropriate regularity conditions are satisfied. In other words, the estimator converges to a minimizer of the expected Kullback-Leibler (KL) divergence over all possible missingness patterns. Remarkably, when the model is well-specified, that is $\P_X = \P_{\theta^*}$ for some $\theta^* \in \Theta$, and the missingness
mechanism is MAR, we recover $\theta_\infty^{\text{ML}} = \theta^*$. 
It turns out that this is deeply related to the KL divergence and does not hold for general M-estimators.


\subsection{Approximated Wasserstein gradient flow without missing values}
\label{sec:wgf}

In this subsection, we summarize the method proposed in \citet{liu2024minimizing} using full data. Suppose we are given an independently and identically distributed sample of complete observations $\{X_i\}_{i=1}^n$ from the distribution $\P_X$ with density $\pX$ and want to create a new sample $\{\tilde{X}_i\}_{i=1}^{\tilde{n}}$ following a distribution $\P_{\tilde{X}}$ with density $\pXnew$, that minimizes the $f$-divergence
\begin{equation*}
    D_f[\pX,\pXnew] = \int \pXnew(x) f(r(x)) \de x \quad \text{with} \quad r(x)=\frac{\pX(x)}{\pXnew(x)}\,.
\end{equation*}
Note that, as a special case, $\KL{\pX \,|\!|\, \pXnew} = D_f[\pX,\pXnew]$ with $f(r)=r\log(r)$. \citet{liu2024minimizing} propose an iterative method to generate such a new sample $\tilde{X}$. More specifically, they use a Wasserstein Gradient Flow (WGF) for the functional $\mathcal{F}[\pXnew] = D_f[\pX,\pXnew]$ with an approximated velocity field. Notably, the $X_i$'s are assumed to be fully observed here. Consequently, the corresponding method is not designed to handle missingness and thus needs to be modified to be applicable in the missing value setup, that we are interested in, as discussed in Section~\ref{sec:method}. 

The WGF iteratively modifies the particles $\{\tilde{X}_{t,i}\}_{i=1}^{\tilde{n}}$, starting with $\{\tilde{X}_{0,i}\}_{i=1}^{\tilde{n}}$ drawn from some initial distribution $\pXnew_0$. Let $\pXnew_t$ denote the distribution of $\tilde{X}_t$. Then, according to \citet[Theorem~2.1]{liu2024minimizing}, briefly explained at the beginning of Section~\ref{sec:method}, the WGF of $\mathcal{F}[\pXnew] = D_f[\pX, \pXnew]$ characterizes the particle evolution via the ordinary differential equation (ODE)
\begin{equation}
\label{eq:wgf_fdiv}
    \frac{\de \tilde{X}_t}{\de t} = \nabla (h \circ r_t)(\tilde{X}_t) \quad \text{with} \quad h(r) = r f'(r) - f(r) \text{ and } r_t(x) = \frac{\pX(x)}{\pXnew_t(x)}\,.
\end{equation}
For the KL divergence, specifically, we have $h(r) = r$. The above ODE moves the density $\pXnew_t$ of $\tilde{X}_t$ along a curve where increasing $t$ results in decreasing $D_f[\pX, \pXnew_t]$, as desired. In practice, the ODE above is discretized using the forward Euler method, yielding
\begin{equation*}
    \tilde{X}_{t+1} = \tilde{X}_{t} + \eta \nabla (h \circ r_t)(\tilde{X}_t) \quad \text{with} \quad t \in \{0, 1, \dots, T-1\}\,,
\end{equation*}
where, by a slight abuse of notation, we use the index $t$ for the discrete time steps as well. $T$ and $\eta$ denote the final time and step size, which both need to be chosen by the user. 

As we do not have access to the densities $\pX$ and $\pXnew_t$ and thus also their ratio $r_t$, the velocity field $\nabla (h \circ r_t)$ needs to be estimated based on the samples $\{X_i\}_{i=1}^n \sim \pX$ and $\{\tilde{X}_{t,i}\}_{i=1}^{\tilde{n}} \sim \pXnew_t$ that we do have access to. For that, \citet{liu2024minimizing} exploit the key observation that
\begin{equation}
\label{eq:key_obs}
    h \circ r_t = \argmax_{\dfunc: \R^d \to \R} \expec[x \sim \pX]{\dfunc(x)} - \expec[x_t \sim \pXnew_t]{\psi^*(\dfunc(x_t))}\,,
\end{equation}
where $\psi$ is chosen such that  $\psi'(r) = r f'(r) - f(r)$ and $\psi^*$ is the convex conjugate of $\psi$. 
For the KL divergence, specifically, we have $\psi(r) = r^2/2$ and $\psi^*(\dfunc) = \dfunc^2/2$. 
Based on \eqref{eq:key_obs}, $h \circ r_t$ is approximated via a local linear estimator. For $w \in \R^d$ and $b \in \R$, let $\dfunc_{w,b}(x) = w^T x + b$ be a linear approximation of~$\dfunc$. Further, let $k_\sigma: \R^d \times \R^d \to \R$ be a given kernel function.
Then, locally, for $x$ close to $x^*$, $(h \circ r_t)(x)$ is approximated by $\dfunc_{w_t(x^*), b_t(x^*)}(x)$, where
\begin{equation}
\label{eq:optim_wb}
    (w_t(x^*), b_t(x^*)) = \argmax_{w \in \R^d,\,b \in \R} \frac{1}{n} \sum_{i=1}^n k_\sigma(X_i, x^*) \dfunc_{w,b}(X_i) - \frac{1}{\tilde{n}} \sum_{i=1}^{\tilde{n}} k_\sigma(\tilde{X}_{t,i}, x^*) \psi^*(\dfunc_{w,b}(\tilde{X}_{t,i}))\,.
\end{equation}
Note the three differences compared to the optimization target in \eqref{eq:key_obs}: (i) The feasible set for $\dfunc$, consisting of all functions from $\R^d$ to $\R$, is restricted to the space of linear functions $\{\dfunc_{w,b} ~|~ w \in \R^d, b \in \R\}$, (ii) the expected values are replaced by their empirical counterparts based on the given samples from $\pX$ and $\pXnew_t$, and (iii) the individual data points in the empirical means are weighted according to the kernel function $k_\sigma$ localizing the approximation.
Finally, $\nabla (h \circ r_t)(x^*)$ is approximated by $\nabla \dfunc_{w_t(x^*), b_t(x^*)} (x^*) = w_t(x^*)$, and the approximated (discretized) particle evolution becomes
\begin{equation*}
    \tilde{X}_{t+1} = \tilde{X}_{t} + \eta w_t(\tilde{X}_t) \quad \text{with} \quad t \in \{0, 1, \dots, T-1\}
\end{equation*}
and $w_t(\tilde{X_t})$ computed via the optimization in \eqref{eq:optim_wb}. 


In the next section, we describe how the method can be modified to be applicable when the data points $X_i$ are not fully observed but contaminated with missing values.

\subsection{Further related literature}
\label{sec:more_literature}



The idea of generating new values from observation with missing values itself is not entirely new. Besides \citet{Bayes}, there are several recent papers developing this idea, including \citet{misgan} and \citet{missdiff}. However, they use GAN and diffusion-based approaches, respectively. Their approaches are thus different from our WGF approach, and in particular, both make use of neural networks for their computation, while our algorithm only uses a tuning-lean local linear approach.

As mentioned above, there is a wealth of imputation methods, which may also be seen as a way of generating complete data. Of particular interest are \citet{neuimp} and \citet{MIRI}. The former uses a WGF for imputation, and thus appears very close to our method at first glance. However, while we are able to use the elegant local linear approach of \citet{liu2024minimizing} averaged over patterns, they are using a neural network and Denoising Score Matching, making their algorithm considerably more complex and prone to tuning parameters. Similarly, while \citet{MIRI} also recognize the importance of minimizing the KL divergence through rectified flows, they need to iteratively decrease and then re-estimate the KL divergence between $\P_{\tilde{X}, M}$ at time $t$ and $\P_{\tilde{X}} \otimes \P_M$ at time $t-1$ and also use neural networks. We instead use the idea of the ignoring MLE to directly minimize the KL divergence between our generated data and the observed part of the real data.

One major motivation for data generation instead of imputation is to be able to obtain (identification) results under MAR, motivated by parametric theory. While there are interesting results in \citet{MIRI} in this direction, they are also somewhat incomplete. In particular, it is unclear whether their optimization target, namely minimizing the mutual information between $X$ and $M$, yields the true distribution under perfect estimation. In contrast, our approach is theoretically motivated by the parametric case, allowing us to more easily obtain theoretical results. For instance, we are able to show that our method can recover the true distribution under perfect estimation, under MAR and an additional positivity condition (see Proposition~\ref{prop:KLmin}). Although this positivity condition is not discussed in \citet{MIRI}, it can be shown that it is necessary. The distribution recovered when it does not hold might still be one respecting MAR, but there is no guarantee that it is the correct one, even under perfect estimation. This is similarly true for the interesting results presented in \citet{missdiff}. While they do not mention MAR, they show that under mild conditions on the missing value mechanism, their objective upper bounds the true negative likelihood for a given pattern. However, this again does not guarantee any identification, even under perfect estimation.

Finally, as noted above, we directly employ tools introduced in \citet{liu2024minimizing} adapted to missing values. In fact, \citet{liu2024minimizing} include an application example for missing values using the WGF to minimize $\text{KL}(\P_{\tilde{X},M} \,|\!|\, \P_{\tilde{X}} \otimes \P_M)$. However, as pointed out in \citet{MIRI}, the WGF is not applicable for such a target as the optimization variable $\tilde{X}$ appears in both densities, $\pX$ and $\pXnew$. In a WGF, however, the distribution $\pX$ is assumed to be fixed. Moreover, it is not clear whether this objective is reasonable under MAR.

%% file: content/method.tex
\section{\name}
\label{sec:method}

We now extend the approximate WGF framework of \citet{liu2024minimizing} to the missing data setting. For that, we turn back to the missing value problem, where we observe the contaminated dataset $(X_i^*, M_i)_{i=1}^n$ with missing values $X^*_{i,j} = \NA$ whenever $M_{i,j} = 1$. As motivated by the parametric case, described in Section~\ref{sec:mle}, we want to generate a new complete sample $\{\tilde{X}_i\}_{i=1}^{\tilde{n}}$ in a way that the corresponding distribution $\P_{\tilde{X}}$ minimizes the expected KL divergence.

We start by providing the following population result, generalizing the case of the parametric MLE.
\begin{proposition}[Population consistency of the KL minimizer]\label{prop:KLmin}
Assume that MAR holds and $\PMmXx{0} > 0$ for almost all $x$. Then $\P_X$ is the unique solution to
\begin{equation}
\label{eq:ML_limit_nonparametric}
   \argmin_{P_X} \expec[M \sim \P_M]{\KL{\P_{X\mid M}^{(M)} \,|\!|\, P_X^{(M)}}}\,.
\end{equation}
\end{proposition}
\noindent
We note that if either of the two conditions does not hold, i.e.\ if either $\PMmXx{0} = 0$ on a set of nonzero probability or the missingness mechanism is not MAR, \eqref{eq:ML_limit_nonparametric} no longer holds. Both might thus be seen as necessary conditions.


Motivated by Proposition \ref{prop:KLmin} and the consistency of the MLE, we now derive a WGF approach to approximately sample from $\P_X$. Let $\mathcal{P}_2(\R^d)$ denote the space of probability measures on $\R^d$ with finite second moment, and let  $\mathcal{F}:\mathcal{P}_2(\R^d) \to \R$ be a functional to be minimized. When $\rho_t$ follows the WGF of $\mathcal{F}$, the corresponding probability flow ODE for particles $X_t \sim \rho_t$ takes the form
\begin{equation*}
    \frac{\de X_t}{\de t} = v_t(X_t) \quad\text{with}\quad v_t(x) = - \nabla \frac{\delta \mathcal{F}[\rho]}{\delta \rho}(x)\,.
\end{equation*}
Here, $\delta \mathcal{F}[\rho]/\delta \rho$ denotes the first $L^2$-variation of $\mathcal{F}$, i.e.\ the function that satisfies
\begin{equation*}
    \frac{\de}{\de \eps}\Big\rvert_{\eps = 0} \mathcal{F}[\rho_\eps] = \int_{\R^d} \frac{\delta \mathcal{F}[\rho]}{\delta \rho}(x) \frac{\partial \rho_\eps(x)}{\partial \eps}\Big\rvert_{\eps=0} \de x
\end{equation*}
for any smooth $\rho_\eps: (-\eps_0, \eps_0) \to \mathcal{P}_2(\R^d)$ with $\rho_0 = \rho$. In case $\mathcal{F}[\rho] = D_f[\pX, \rho]$ 
is an $f$-divergence, it can easily be verified that
\begin{equation*}
    \frac{\delta \mathcal{F}[\rho]}{\delta \rho}(x) = f\qty(\frac{\pX(x)}{\rho(x)}) - f'\qty(\frac{\pX(x)}{\rho(x)}) \frac{\pX(x)}{\rho(x)} = -(h \circ r)(x)\,,
\end{equation*}
which gives rise to the particle ODE \eqref{eq:wgf_fdiv}.
For our objective, we get the following result. 
\begin{lemma}[First variation of the objective]
\label{lem:deriv}
    For $\mathcal{F}[\pXnew] = \expec[M \sim \P_M]{D_f(\pX_M^{(M)}, \pXnew^{(M)})}$, it holds that
    \begin{equation*}
        \frac{\delta \mathcal{F}[\pXnew]}{\delta \pXnew}(x) = - \expec[M \sim \P_M]{(h \circ r^{(M)})(x^{(M)})} \quad\text{with}\quad r^{(M)} = \frac{\pX_M^{(M)}}{\pXnew^{(M)}}\,.
    \end{equation*}
\end{lemma}
\noindent
The proof can be found in Appendix \ref{sec:proofs}.
It follows that, in the missing value setup, the particle ODE is given by 
\begin{equation}
\label{eq:ode_cont}
    \frac{\de X_t}{\de t} = \nabla \expec[M \sim \P_M]{(h \circ r_t^{(M)})(X_t^{(M)})}
\end{equation}
with $h(r)=r$ in the special case of the KL divergence.
Next, we derive a result similar to the key observation \eqref{eq:key_obs} of \citet{liu2024minimizing}, that builds the backbone of our method.

\begin{lemma}[Variational form for $h \circ r$]
\label{lem:key_obs_m}
    For each missingness pattern $m \in \mathcal{M}$, it holds that
    \begin{equation}
    \label{eq:key_obs_m}
        h \circ r^{(m)} = \argmax_{\dfunc_m:\, \R^{d_m} \to \R}
        \left\{ \expec[x^{(m)} \sim \pX_m^{(m)}]{\dfunc_m(x^{(m)})} - \expec[x^{(m)} \sim \pXnew^{(m)}]{\psi^*(\dfunc_m(x^{(m)}))} \right\}\,,
    \end{equation}
    where $d_m = d-\sum_{j=1}^{d} m_j$ denotes the number of observed values under missingness pattern $m$. 
\end{lemma}
\noindent
The proof is again deferred to the appendix.
Similar to \citet{liu2024minimizing}, we estimate the maximizer in \eqref{eq:key_obs_m} by using a local linear approximation for $\dfunc_m$, replacing the expectations by their empirical counterparts and weighting the individual summands with a kernel function. In our specific case, we 
let $\mathcal{M}_n = \{m \in \mathcal{M} \mid \exists i \in \{1, \dots, n\} \text{ s.t. } M_i =m \}$ denote the set of observed patterns.
For each $m \in \mathcal{M}_n$, we define $\{X_{m,i}^{(m)}\}_{i=1}^{n_m} = \{X_i^{*(m)} \mid M_i=m,~i\in\{1,\dots,n\}\}$ to be the subset of observations with missingness pattern $m$. This can be interpreted as a sample from the conditional distribution $\P_{X\mid M=m}^{(m)}$ with density $\pX_m^{(m)}$. Note that the sample size $n_m$ might differ for each missingness pattern $m$, summing up to $n$ in total. 
Moreover, for each $m \in \mathcal{M}$, $\{\tilde{X}^{(m)}_{t,i}\}_{i=1}^{\tilde{n}}$ is a sample from $\P_{\tilde{X}_t}^{(m)}$ with density $\rho_t^{(m)}$. A local linear approximation of the maximizer $(h \circ r_t^{(m)})(x)$ for $x$ close to $x^*$ is then given by $g_{w,b}(x)=w^T x + b$ with the optimal choice of $(w,b)$ defined as 
\begin{align}
\label{eq:optim_wb_m}
    (w_{t,m}^{(m)}(x^*), b_{t,m}(x^*)) &= \argmax_{w \in \R^{d_m}, b \in \R} \frac{1}{n_m} \sum_{i=1}^{n_m} k_\sigma(X_{m,i}^{(m)}, {x^*}^{(m)}) \dfunc_{w,b}(X_{m,i}^{(m)}) \nonumber\\
    &\qquad\qquad\quad- \frac{1}{\tilde{n}} \sum_{i=1}^{\tilde{n}} k_\sigma(\tilde{X}^{(m)}_{t,i}, {x^*}^{(m)}) \psi^*(\dfunc_{w,b}(\tilde{X}_{t,i}^{(m)}))\,.
\end{align}
For the KL divergence, we use $\psi^*(\dfunc) = \dfunc^2/2$. Consequently, the gradient $\nabla (h \circ r_t^{(m)})(x^*)$ can be estimated by $\nabla g_{w_{t,m}(x^*), b_{t,m}(x^*)}(x^*) = w_{t,m}(x^*)$, whereby $w_{t,m}(x^*)$ is the same as $w_{t,m}^{(m)}(x^*)$ for any index $j$ with $m_j=0$ and zero otherwise.
Using this approximation in \eqref{eq:ode_cont} and replacing the expectation over $M$ by its empirical counterpart, we finally get the approximated discretized WGF particle ODE
\begin{equation}\label{eq:Mainiteration}
    \tilde{X}_{t+1} = \tilde{X}_{t} + \eta \sum_{m \in \mathcal{M}_n} \frac{n_m}{n} w_{t,m}(\tilde{X}_t) \quad \text{with} \quad t \in \{0, 1, \dots, T-1\}\,,
\end{equation}
starting with $\tilde{X}_0$ drawn from some initial distribution $\pXnew_0$, which, together with the step size $\eta$ and the number of iterations $T$, needs to be chosen by the user.

The following upper bound for the estimation error of $w_{t,m}(\cdot)$ 
justifies the local linear approximation.
\begin{proposition}[Velocity field approximation error]
\label{prop:approx_error}
    Under the assumptions of \citet[Theorem 4.8]{liu2024minimizing} and if $\sigma \le 1$, it holds that, for all $t \in \{0, 1, \dots, T-1\}$ and $x^* \in \mathcal{X}$,
    \begin{equation}
    \label{eq:err_bound}
        \left|\!\left|\sum_{m \in \mathcal{M}} \frac{n_m}{n} w_{t,m}(x^*) - \expec[M]{\nabla (h \circ r_t^{(M)})(x^*)}\right|\!\right|_2 \le \bigo{  \sqrt{\frac{\abs{\mathcal{M}}}{n \sigma^d}} + \sigma^2}\,,
    \end{equation}
    where $\abs{\mathcal{M}}$ denotes the cardinality of $\mathcal{M}$, that is, the number of patterns with positive probability.
\end{proposition}
\noindent
The proof is deferred to the appendix, and we refer to \citet{liu2024minimizing} for the precise assumptions and constants hidden in the $\mathcal{O}$-notation. Consequently, the estimation error vanishes as $n$ tends to infinity and $\sigma$ goes to zero, provided $n\sigma^d$ diverges. Moreover, the bound in \eqref{eq:err_bound} illuminates the dependence of the estimation error on the cardinality of $\mathcal{M}$: The more missingness patterns have nonzero probability, the worse the constant in the rate gets, with $\abs{\mathcal{M}}=2^d$ in the worst case. It is also instructive to note that $\abs{\mathcal{M}}$ could be replaced by $\abs{\mathcal{M}_n} \leq \abs{\mathcal{M}}$, the actual number of patterns drawn for a given $n$. This illustrates that the rate is not dependent on the minimum number of observations across patterns, but rather on the overall number of patterns relative to $n$. In particular, having a small number of observations $n_m$ for some patterns $m$ will not diminish the rate as long as $\abs{\mathcal{M}_n}$ is small relative to $n$. Although the result holds for all points $x^*$ and time steps $t$, the coefficients in the upper bound on the right-hand side of \eqref{eq:err_bound} may well depend on their values. Establishing a uniform bound remains an interesting direction for future work. We also note that, thanks to the MAR assumption and Lemma \ref{lem:deriv}, it is possible to focus only on the observed part of a given pattern $m$. This is what makes the approximation result possible, even under complex non-monotone missingness, where each $X \mid M=m$ may follow a different distribution.


\citet{liu2024minimizing} propose to use gradient descent for the optimization in \eqref{eq:optim_wb}. However, since we are specifically interested in the KL divergence, where $\psi^*(\dfunc)=\dfunc^2/2$ has a linear derivative $(\psi^*)'(d) = d$, the optimization problem reduces to a system of linear equations that can be solved directly. The advantage of using a direct numerical solver for this system of linear equations instead of a gradient descent is threefold: (i) It results in a higher accuracy of the solution, (ii) it is computationally more efficient and (iii) it reduces the number of hyperparameters, as, for gradient descent, the user needs to choose the step size and the number of steps (or another termination condition). Derivations of the linear approach, as well as further details on the implementation are given in Appendix \ref{sec:direct_solve}. This appendix also includes pseudo-code in Algorithm~\ref{alg:method}. 

Finally, we note that our method has only three hyperparameters: the time step size $\eta$, number of time steps $T$, and bandwidth $\sigma$. Importantly, the choice of $T$ and $\eta$ does not involve a trade-off. In principle, smaller values of $\eta$ and larger values of $T$ should always improve performance, with the only practical limitation being the available computational resources. By contrast, $\sigma$ must be selected more carefully, as both excessively large and small values can degrade the results. Nevertheless, as shown in the empirical section, a simple heuristic choice of $\sigma$ already yields strong performance.

%% file: content/results.tex
\section{Empirical results}
\label{sec:emp_res}

In this section, we demonstrate that \name can match or even exceed state-of-the-art imputation methods in MAR examples. 
As competitors we include MICE \citep{mice}, GAIN \citep{GAIN}, Hyperimpute \citep{Jarrett2022HyperImpute}, MIRI \citep{MIRI}, the Bayesian approach of \citet{Bayes}, MissDiff \citep{missdiff} and NewImp \citep{neuimp}. We compare these methods on a simulated example, originally devised to test imputations under MAR, as well as ten real datasets. In all experiments, we use the RBF kernel as $k_\sigma$, take $T=1000$, $\eta=0.01$, and $\sigma$ following the median heuristic of \citealt{gretton2012optimal} based on the initial sample $\tilde{X}_0$.



\subsection{Simulation study}
\label{sec:toyexmpl}

The following example, inspired by \citet[Example 6]{näf2024goodimputationmarmissingness}, is designed to test imputation methods for their handling of MAR.
For $d=3$, we take $X$ to have uniform marginals on $[0,1]$. Using a copula, we induce dependence between $X_1, X_2$, while $X_3$ is independent of both. We then define the following missingness mechanism: For $(m_1,m_2, m_3)=((0,0,0), (0,1,0), (1,0,0))$, let
    \begin{align*}
         &\PMmXx{m_1}=(x_1+x_2)/3\,; \nonumber \\
         &\PMmXx{m_2}=(2-x_1)/3\,; \nonumber \\
         &\PMmXx{m_3}=(1-x_2)/3\,.
    \end{align*}
This mechanism clearly meets MAR and, moreover, $\P(M=0 \mid X=x) > 0$ for all $x$, as required by Proposition \ref{prop:KLmin}. Nonetheless, as outlined in \citet{näf2024goodimputationmarmissingness}, this is a rather complex non-monotone mechanism, designed to be difficult for imputation and estimation. The challenge formulated in \citet{näf2024goodimputationmarmissingness, practical} was to estimate the $0.1$ quantile of $X_1$. We note that any other base distribution $\P_X$ could be used for this experiment. In Appendix \ref{sec:additionalresults}, we perform the same experiment using a multivariate Gaussian distribution instead.

We simulate the above data generating process for a sample size of $n=2000$ and $B=20$ replications and attempt to generate a sample from $\P_X$ with our \name and the aforementioned imputation methods. For our method and MIRI, we initialize the particle evolution at $\tilde{X}_0 = X^*$ with missing values replaced by samples from the observed values in the same column. Figure \ref{fig:UNIboxplot} shows the result for both the quantile estimation of $X_1$ (right) as well as the energy distance \citep{EnergyDistance} between an independent sample and the recreated distribution (left). We can see that \name meets the challenge remarkably well in both metrics, far surpassing GAIN, Hyperimpute, MIRI, the Bayesian approach, MissDiff and NewImp. The only method coming close is MICE, which tends to empirically work exceptionally well in continuous MAR settings, but it is still less accurate than \name. In particular, our approach is arguably the method that estimates the true quantile most accurately, a rather impressive feat given the difficulty of the MAR mechanism and the strength of the competitors. Figure \ref{fig:UNIimputations} in the Appendix also confirms this visually, showing that \name recreates the distribution most accurately. Though this is merely a toy example, one would be hard-pressed to find a method that performs this well.

\begin{figure}
    \centering
    \includegraphics[width=0.9\linewidth]{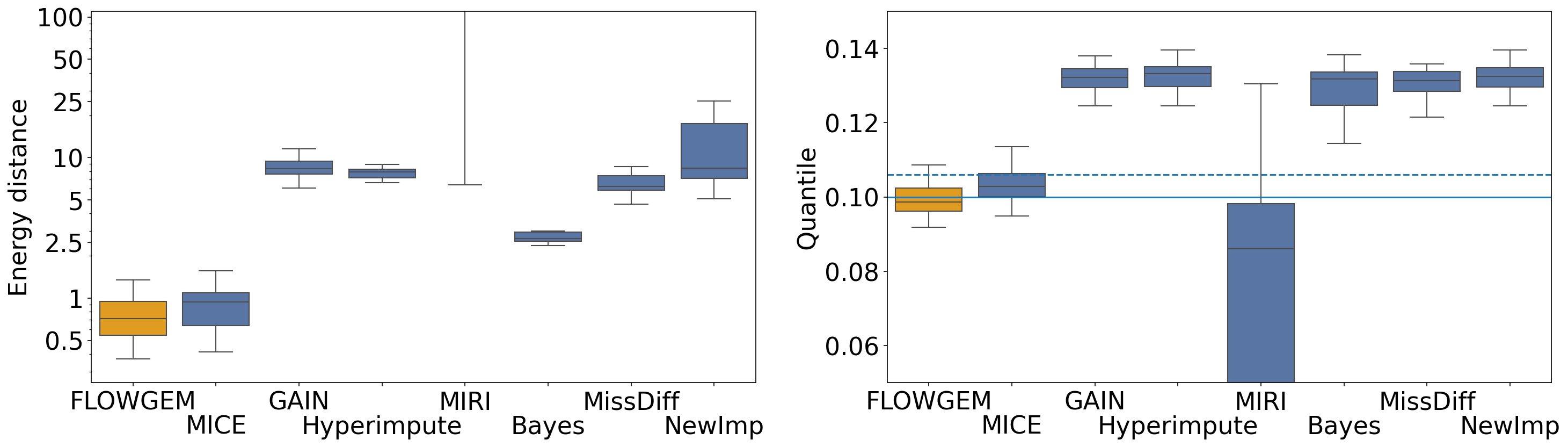}
    \caption{Standardized Energy distance in log-scale (left) and quantile estimate (right) for the simulated example with uniform distribution and $n=2000$, repeated over $B=20$ times. The solid blue line on the right indicates the true quantile value of the uncontaminated distribution, while the dashed line shows the population quantile when missing values are ignored. Our method, \name, clearly outperforms its competitors, minimizing the energy distance and estimating the quantile most accurately. We note that MICE also performs well, but unlike \name, it lacks theoretical justification. 
    Some values are cut off for readability: the energy distance for MIRI ranges from $6$ to $4 \cdot 10^{148}$ with a median of $2 \cdot 10^{46}$, and its quantile estimates reach as low as $-2.95$.}
    \label{fig:UNIboxplot}
\end{figure}

\subsection{Real examples}
\label{sec:realdata}

In this section, we consider ten real datasets of varying size and dimension with simulated missingness. The datasets are taken from various sources, such as the UCI machine learning repository \citep{lichman2013uci}. Table \ref{tab:datasetsfull} in Appendix \ref{sec:additionalresults} gives a detailed account of the source of each dataset. We randomly split each dataset in two parts and generate MAR missing values for the first part, using random patterns and the ampute function from the \textsf{R} package \texttt{mice}. This has become a standard way of introducing MAR and leads to around 40--50\% of missing values. For details we refer to Appendix \ref{sec:additionalresults}. Given this dataset with missing values, we approximately generate the full data either using \name or one of the competing methods and then compare the result to the second (independent) sample using the standardized energy distance described in Appendix \ref{sec:additionalresults}. Table \ref{tab:datasets} shows the result. Crucially, in all ten datasets, our method displays the lowest or second lowest energy value, sometimes by a large margin compared to the other methods. This is despite the large range in number of observations and dimensions, ranging from very small ($n=215$, $d=6$) to very large ($n=4901$, $d=272$). Unfortunately, the method of \citet{Bayes} displayed infeasible runtime for the larger datasets, as we detail in Appendix \ref{sec:additionalresults}.

\begin{table}
\caption{Standardized energy distance between generated and held-out samples across ten real-world datasets (lower is better). ``--'' indicates the method exceeded the runtime limit. Values larger than $10^3$ are truncated and displayed as $>10^3$. For each dataset, the best result is marked in bold, while the second best result is underlined. Our method is consistently either best or second best.} 
\label{tab:datasets}
\centering
\small
\begin{tabular}{lrrrrrrrr}
\toprule
Dataset & Ours & MICE & GAIN & Hyperimp. & MIRI & Bayes & MissDiff & NewImp \\
\midrule
scm1d & \underline{26.26} & \textbf{25.53} & $>10^3$ & 31.17 & $>10^3$ & -- & 457.23 & $>10^3$ \\
scm20d & \textbf{17.65} & \underline{21.40} & $>10^3$ & 31.45 & $>10^3$ & -- & 226.51 & $>10^3$ \\
pumadyn32nm & \underline{11.41} & \textbf{11.20} & $>10^3$ & 116.12 & $>10^3$ & -- & 230.36 & $>10^3$ \\
parkinsons & \textbf{22.98} & \underline{29.21} & $>10^3$ & 31.53 & $>10^3$ & -- & 44.99 & $>10^3$ \\
allergens & \underline{34.29} & 209.30 & 205.81 & \textbf{30.05} & 125.09 & -- & 89.17 & $>10^3$ \\
concrete & \textbf{5.84} & 22.28 & 29.84 & \underline{12.79} & 20.09 & $>10^3$ & 20.21 & $>10^3$ \\
stock & \underline{4.70} & \textbf{4.30} & 421.85 & 7.81 & 67.41 & $>10^3$ & 7.28 & $>10^3$ \\
forest & \textbf{3.99} & 4.68 & 4.55 & \underline{4.00} & 9.36 & 7.32 & 4.33 & $>10^3$ \\
housing & \textbf{7.80} & \underline{8.10} & 190.55 & 23.39 & 19.78 & 33.22 & 13.31 & $>10^3$ \\
windspeed & \underline{2.03} & 2.10 & 2.31 & \textbf{1.95} & 15.34 & 46.29 & 2.77 & $>10^3$\\
\bottomrule
\end{tabular}
\end{table}

%% file: content/conclusion.tex
\section{Conclusion}
\label{sec:concl}
In this paper, we introduced \name, a principled sample generator under MAR based on an approximate Wasserstein gradient flow. The method is grounded in a deep theoretical justification: the target functional, minimizing the average KL divergence over patterns, recovers the true distribution at the population level under MAR (Proposition~\ref{prop:KLmin}). Further, the approximation error of the local linear velocity field estimator vanishes asymptotically (Proposition~\ref{prop:approx_error}), an important first step in ensuring theoretical validity in the finite-sample regime. Empirically, \name
achieves impressive performance in MAR examples that rivals and often far surpasses modern imputation methods (Figure~\ref{fig:UNIboxplot} and Table~\ref{tab:datasets}), 
without relying on computationally expensive or difficult-to-tune neural networks. Notably, for all experiments, ranging from simulated data with $d=3$ to real data with widely varying scales up to $d=272$, the same values of $\eta$ and $T$, with a simple heuristic for $\sigma$, were used. In summary, \name\ shows that theoretical rigor and strong empirical performance can go hand in hand, a combination that remains rare among existing imputation methods.

This work opens several interesting directions of further research. First, while our principled approach already admits some theoretical analysis, establishing consistency of the distributional estimate in terms of a distributional distance is an important next step. Second, our algorithm is already fast for dimensions $d < 100$ but might require further improvements for higher dimensions. Third, extending \name to categorical or mixed-type data is a natural and practically important next step. 

%% file: content/appendix.tex
\section{Appendix}

\subsection{Algorithm and additional details}
\label{sec:direct_solve}

In this section, we first discuss the system of linear equations that is crucial to our algorithm and then present pseudo-code for our procedure in Algorithm \ref{alg:method}. In addition, we provide extensive details on the implementation.

\paragraph{Linear System of Equations} \citet{liu2024minimizing} propose to use gradient descent for the optimization in \eqref{eq:optim_wb}. Of course, we could use the same technique to approximate the solution of \eqref{eq:optim_wb_m}. However, since we are specifically interested in the KL divergence and, in that case, $\psi^*(\dfunc)=\dfunc^2/2$ has a linear derivative $(\psi^*)'(d) = d$, the optimization problem reduces to a system of linear equations that can be solved directly. 

Specifically, the objective function for a fixed $x^*$ and $m$ is given by 
\begin{equation*}
    L(w,b) = \frac{1}{n_m} \sum_{i=1}^{n_m} k_\sigma^{(m)}(X_{m,i}) \qty((X_{m,i}^{(m)})^T w + b) - \frac{1}{\tilde{n}} \sum_{i=1}^{\tilde{n}} k_\sigma^{(m)}(\tilde{X}_{t,i}) \psi^*\qty((\tilde{X}_{t,i}^{(m)})^T w + b)\,,
\end{equation*}
where we use the abbreviation $k_\sigma^{(m)}(x) \coloneqq k_\sigma(x^{(m)}, {x^*}^{(m)})$. To find the maximum, we solve
\renewcommand*{\arraystretch}{1.5}
\begin{equation*}
    \nabla L(w, b) = \begin{pmatrix}
        \frac{1}{n_m} \sum_{i=1}^{n_m} k_\sigma^{(m)}(X_{m,i}) X_{m,i}^{(m)} - \frac{1}{\tilde{n}} \sum_{i=1}^{\tilde{n}} k_\sigma^{(m)}(\tilde{X}_{t,i}) \qty((\tilde{X}_{t,i}^{(m)})^T w + b) \tilde{X}_{t,i}^{(m)}\\
        \frac{1}{n_m} \sum_{i=1}^{n_m} k_\sigma^{(m)}(X_{m,i}) - \frac{1}{\tilde{n}} \sum_{i=1}^{\tilde{n}} k_\sigma^{(m)}(\tilde{X}_{t,i}) \qty((\tilde{X}_{t,i}^{(m)})^T w + b)
    \end{pmatrix}
    = 0\,,
\end{equation*}
which can be rewritten as a system of linear equations $A(w,b)^T = c$ with 

\begin{equation*}
    A = \begin{pmatrix}
        A_{11} & A_{21}^T\\
        A_{21} & A_{22}
    \end{pmatrix}
    \quad \text{and} \quad
    c = \begin{pmatrix}
        c_{1}\\
        c_{2}
    \end{pmatrix}\,,
\end{equation*}
where
\begin{align}
    A_{11} &= \begin{pmatrix}
        \frac{1}{\tilde{n}} \sum_{i=1}^{\tilde{n}} k_\sigma^{(m)}(\tilde{X}_{t,i}) \tilde{X}_{t,i,1}^{(m)} \tilde{X}_{t,i,1}^{(m)} & \cdots &\frac{1}{\tilde{n}} \sum_{i=1}^{\tilde{n}} k_\sigma^{(m)}(\tilde{X}_{t,i}) \tilde{X}_{t,i,1}^{(m)} \tilde{X}_{t,i,d_m}^{(m)}\\
        \vdots & \ddots &\vdots\\
        \frac{1}{\tilde{n}} \sum_{i=1}^{\tilde{n}} k_\sigma^{(m)}(\tilde{X}_{t,i}) \tilde{X}_{t,i,d_m}^{(m)} \tilde{X}_{t,i,1}^{(m)} & \cdots &\frac{1}{\tilde{n}} \sum_{i=1}^{\tilde{n}} k_\sigma^{(m)}(\tilde{X}_{t,i}) \tilde{X}_{t,i,d_m}^{(m)} \tilde{X}_{t,i,d_m}^{(m)}
    \end{pmatrix}\,, \nonumber\\
    A_{21} &= \begin{pmatrix}
        \frac{1}{\tilde{n}} \sum_{i=1}^{\tilde{n}} k_\sigma^{(m)}(\tilde{X}_{t,i}) \tilde{X}_{t,i,1}^{(m)} &\dots &\frac{1}{\tilde{n}} \sum_{i=1}^{\tilde{n}} k_\sigma^{(m)}(\tilde{X}_{t,i}) \tilde{X}_{t,i,d_m}^{(m)}
    \end{pmatrix}\,, \nonumber\\
    A_{22} &= \frac{1}{\tilde{n}} \sum_{i=1}^{\tilde{n}} k_\sigma^{(m)}(\tilde{X}_{t,i})\,, \label{eq:A}
\end{align}
and
\begin{align}
    c_1 &= \begin{pmatrix}
        \frac{1}{n_m} \sum_{i=1}^{n_m} k_\sigma^{(m)}(X_{m,i}) X_{m,i,1}^{(m)} & 
        \dots &
        \frac{1}{n_m} \sum_{i=1}^{n_m} k_\sigma^{(m)}(X_{m,i}) X_{m,i,d_m}^{(m)}\\ 
    \end{pmatrix}^T, \nonumber\\
    c_2 &= \frac{1}{n_m} \sum_{i=1}^{n_m} k_\sigma^{(m)}(X_{m,i})\,. 
    \label{eq:c}
\end{align}
In practice, we apply Tikhonov regularization to $A$, replacing it with $A + \varepsilon I$ where $\varepsilon = 10^{-5}$, to ensure numerical stability during inversion.

\renewcommand*{\arraystretch}{1}

\begin{algorithm}
    \caption{\name (FLOW-based GEneration for Missing data)}
    \label{alg:method}

    \SetKwFunction{SGD}{SGD\_update}
    \SetKwFunction{Init}{initial\_sample}
    \SetKwFunction{Unique}{unique\_rows}
    \SetKwFunction{Mean}{mean\_over\_rows}
    \SetKwFunction{ComputeA}{compute\_A}
    \SetKwFunction{ComputeC}{compute\_c}
    \SetKwFunction{Solve}{solve}

    \KwData{Observations with missing values $(\mathbf{X}, \mathbf{M}) \in \R^{n\times d} \times \{0,1\}^{n \times d}$,\\
    Method to sample from $\pXnew_0$ (\Init), Final time $T$, Step size $\eta$,\\
    Kernel function $k_\sigma$, Bandwidth $\sigma$,
    New sample size $\tilde{n}$}
    \KwResult{New complete sample $\{\tilde{X}_{T,i}\}_{i=1}^{\tilde{n}}$ approximately distributed as $\P_X$}

    $\mathbf{\tilde{X}}_0 \leftarrow$ \Init{$\mathbf{X}, \mathbf{M}, \tilde{n}$} \tcp{$\in \R^{\tilde{n} \times d}$}
    \For{$t \leftarrow 0$ \KwTo $T-1$}{
        \For(\tcp*[h]{can be parallelized or vectorized}){$i \leftarrow 0$ \KwTo $\tilde{n} - 1$}{
            $x^* \leftarrow \mathbf{\tilde{X}}_{t,i}$\;
            \For(\tcp*[h]{can be parallelized}){$m \in$ \Unique{$\mathbf{M}$}}{
                $\mathbf{X}_m \leftarrow$ rows $\mathbf{X}_j$ with $\mathbf{M}_j = m$\;
                $A \leftarrow$ \ComputeA($k_\sigma, m, x^*, \mathbf{\tilde{X}}_t$) \tcp{according to \eqref{eq:A}}
                $c \leftarrow$ \ComputeC($k_\sigma, m, x^*, \mathbf{X}_m$) \tcp{according to \eqref{eq:c}}
                $(w_m, b_m) \leftarrow$ \Solve($A$, $c$)
            }
            $\mathbf{\tilde{X}}_{t+1,i} \leftarrow \mathbf{\tilde{X}}_{t,i} + \eta \sum_{m}\frac{n_m}{n}w_m$\;
        }
    }
    \Return{$\mathbf{\tilde{X}}_T$}
\end{algorithm}


\paragraph{Standardization} Some datasets such as ``scm20d'' have widely different scales for different variables. To minimize numerical issues, we standardize the columns by using a diagonal matrix $\Lambda$ and an estimated mean $\mu$. To motivate this approach, we are able to make use of the well-known fact that the KL divergence is invariant to affine transformations. For all $m$ and all $t$, let $\Lambda_m=(\Lambda_{i,j})_{i: m_i=0, j: m_j=0}$ be the $d_m$-dimensional submatrix of $\Lambda$ and similarly, $\mu_m=(\mu_j)_{j: m_j=0}$. Moreover, let 
\[
Y^{(m)} = \Lambda^{-1/2}_{m} (X^{(m)}-\mu_m) \quad \text{and}\quad \tilde{Y}_t= \Lambda^{-1/2} (\tilde{X}_t-\mu)
\]
be the standardized data, and define by
\begin{align*}
   \tilde{\pi}_m^{(m)}(y^{(m)}) &= \det(\Lambda_m)^{1/2} \pi_m^{(m)}(\Lambda_m^{1/2} y^{(m)} + \mu),\\
   \tilde{\rho}_t(y) &= \det(\Lambda)^{1/2}  \rho_t(\Lambda^{1/2} y + \mu)
\end{align*} 
the respective scaled densities. Then by a simple change-of-variable argument in the integration, it holds that
\begin{align*}
    \expec[M \sim \P_M]{\KL{\P_{X|M}^{(M)} \,|\!|\, \P_{\tilde{X}_t}^{(M)}}}\quad= \sum_{m \in \mathcal{M}} \P(M=m) \int \log\left( \frac{\pi_m^{(m)}(x^{(m)})}{\rho_t^{(m)}(x^{(m)})}  \right)\pi_m^{(m)}(x^{(m)}) dx^{(m)}\\
    = \sum_{m \in \mathcal{M}} \P(M=m) \int \log\left( \frac{\tilde{\pi}_m^{(m)}(y^{(m)})}{\tilde{\rho}_t^{(m)}(y^{(m)})}  \right)\tilde{\pi}_m^{(m)}(y^{(m)}) dy^{(m)}\,.
\end{align*}
The last expression is simply the KL divergence between the scaled distributions. Combined with Proposition \ref{prop:KLmin}, this shows that minimizing the KL divergence between the scaled distribution will recover $P_X$ after rescaling. Thus, we may standardize $X$ and $X_t$ using any $\mu$, $\Lambda$. In our case, they are obtained by calculating standard deviations and means separately for each column only on the observed data.

\paragraph{Early Stopping} In the real data experiments, we employ an early stopping mechanism. In particular, we measure the relative change of the particles in \eqref{eq:Mainiteration} at each $t$ and stop if
\begin{align*}
    \frac{\frac{1}{\tilde{n}} \sum_{i=1}^{\tilde{n}} |\!|\sum_{m \in \mathcal{M}_n} \frac{n_m}{n} w_{t,m}(\tilde{X}_{t,i}) |\!|_2}{\frac{1}{\tilde{n}} \sum_{i=1}^{\tilde{n}} |\!|\tilde{X}_{t,i}|\!|_2} < \varepsilon\,,
\end{align*}
using $\varepsilon = 0.01$. This not only makes the algorithm faster, but also impedes the algorithm from overshooting in case $\eta$ was chosen too large. Finally, if the norm of the gradient increases in an iteration, we halve $\eta$ for all subsequent iterations.  


\begin{table}
\caption{Real data examples considered in Section~\ref{sec:realdata}. For each dataset, we report the number of rows, the number of columns, the actual number of observations $n$ and dimensions $d$ used in the analysis, as well as their source. UCI refers to \citet{lichman2013uci} and was downloaded from \protect\url{https://github.com/treforevans/uci_datasets/tree/master}.} 
\label{tab:datasetsfull}
\centering
\begingroup
\begin{tabular}{lccccl}
  \hline
Name & Rows & Cols & n & d & Source \\ 
  \hline
scm1d & 9803 & 296  &  4901 & 272 & \citet{OpenML2013} \\ 
scm20d & 8966 & 77 & 4482 &  58 &  \citet{OpenML2013} \\ 
pumadyn32nm & 8192  & 33   & 4095 &  33 & UCI\\
parkinsons & 5875  & 21 & 2937 &  18   & UCI\\
allergens & 2351 & 112   & 1175 &  49 &  \citet{allergens_datachallenge} \\
  concrete & 1030 & 9 & 514 &   8   & \citet{concrete_compressive_strength_165} \\
    stock & 536  & 12  & 267 &   9 & UCI\\
forest & 517 & 13 & 257 &   7  & UCI\\
  housing & 506   & 14  & 252 &  10  & UCI\\
  windspeed & 433 & 6& 215 &   6   &  \citet{R_mice} \\
\hline
\end{tabular}
\endgroup
\end{table}

\begin{table}
    \caption{Average computation time (in seconds) per method for the simulation study of Section~\ref{sec:toyexmpl} with $n=2000$, $d=3$, and uniform distribution. Mean and standard deviation over $B=20$ replications are reported.}
    \label{tab:runtimes}
    \centering
    \small
    \begin{tabular}{lrr}
    \toprule
    Method & Mean & Std. Dev. \\
    \midrule
    FLOWGEM & 22.85 & 1.36 \\
    MICE & 0.56 & 0.00 \\
    GAIN & 2.80 & 0.64 \\
    Hyperimpute & 8.68 & 5.63 \\
    MIRI & 1611.91 & 2588.70 \\
    Bayes & 2038.28 & 520.23 \\
    MissDiff & 19.22 & 1.03 \\
    NewImp & 18.16 & 0.56 \\
    \bottomrule
    \end{tabular}
\end{table}

\subsection{Details on experiments and additional results}\label{sec:additionalresults}

Table \ref{tab:datasetsfull} contains more details on the datasets and their sources. We note that the values of $n$ and $d$ are different than the total number of rows and columns in each dataset, because (1) columns with less than 10\% of unique values were removed for the final dataset and (2) we split the dataset in two for training and test set. 

All computations were performed using Google Colaboratory Pro with a TPU v6e-1 accelerator. The host runtime used Python 3.11. Experiments were run between March 2026 and April 2026. Average computation times per method for the simulated datasets are reported in Table~\ref{tab:runtimes}.

\paragraph{Missing Value Generation} We use the ampute function of the \texttt{mice} \textsf{R} package, as follows: We randomly generate a set of patterns, including a fully observed pattern, and generate MAR missingness with the ampute function according to those patterns, such that every pattern is observed with equal probability. We note that this results in about 40--50\% missing values, calculated as total number of missing entries divided by $n\times d$, as in \citet{OneBenchmarktorulethemall}. We note that it is often left ambiguous in other papers how ``x\% of missing values'' is defined. In our way of measuring, 40\% of missing values is substantial, in fact the recent benchmark by \citet{OneBenchmarktorulethemall} only considers up to 30\%.

\paragraph{Standardized Energy Distance} We assess the performance of imputation methods by measuring the distributional distance between the generated dataset and a set of independent complete data. To this end, we use the energy distance from \citet{EnergyDistance}, defined as
\begin{equation*}
    e^2(X, Y) = 2\mathbb{E} |\!|X - Y|\!|_{2} - \mathbb{E}|\!|X - X'|\!|_{2} - \mathbb{E}|\!|Y - Y'|\!|_{2}\,,
\end{equation*}
where $X$ and $Y$ are random variables (here observed and generated).  Before computing the energy distance, we standardized all variables in both imputed and fully observed datasets using the column-wise means and standard deviations estimated from the fully observed dataset \emph{prior to imposing missingness}. This ensures that the distance measure is not dominated by variables with large natural variance and is comparable across datasets.

\paragraph{Continuous Features} As our method intrinsically assumes the distribution $\P_X$ to be continuous (with density $\pi$), it was designed to handle continuous features. To respect that in our experiments, we filtered out columns in the datasets that had less than 10\% of unique values for our application. 

\paragraph{Initialization} For real data, we generate $\tilde{X}_0$ by imputing the dataset using the MICE implementation in the Python \texttt{hyperimpute} package with 10 iterations (the full MICE implementation uses 100 iterations). This gives a reasonable and quick-to-calculate data-adaptive starting value for our algorithm. 

\paragraph{Competitors} We compare our method against MICE \citep{mice}, GAIN \citep{GAIN}, Hyperimpute \citep{Jarrett2022HyperImpute}, MIRI \citep{MIRI}, the Bayesian density estimation of \citet{Bayes}, MissDiff \citep{missdiff} and NewImp \citep{neuimp}. For GAIN, Hyperimpute and MICE, we use the \texttt{hyperimpute} Python package \citep{hyperimpute-package} with default hyperparameter settings. For MIRI, MissDiff and NewImp, we use the implementation available at \url{https://github.com/yujhml/MIRI-Imputation} (last accessed in March 2026). As the original implementation only supports float32, we adapted the code to float64 to ensure a fair comparison with the remaining methods. For MIRI, we use the hyperparameters specified in \citet{MIRI}, namely \texttt{max\_rounds=15}, \texttt{batchsize=500}, \texttt{maxepochs=900}, \texttt{odesteps=100}, while for MissDiff and NewImp, we use the default settings. We note that, for essentially all competing methods, performance may be improved through more extensive hyperparameter tuning. In particular, \citet{MIRI} mention a different specification with higher number of rounds. However, our current setting was already straining our computational resources. While computation times of the current MIRI code are remarkably resistant to higher $d$, completing in about 70 minutes on the ``scm1d'' dataset, the code was excruciatingly slow on the smaller datasets. This can be seen for instance in the simulation with $d=3$ in Table~\ref{tab:runtimes}. Moreover, we saw no obvious convergent behavior and instead the code usually started producing \texttt{nan}s for too many iterations.

\begin{figure}
    \centering
    \includegraphics[width=0.9\linewidth]{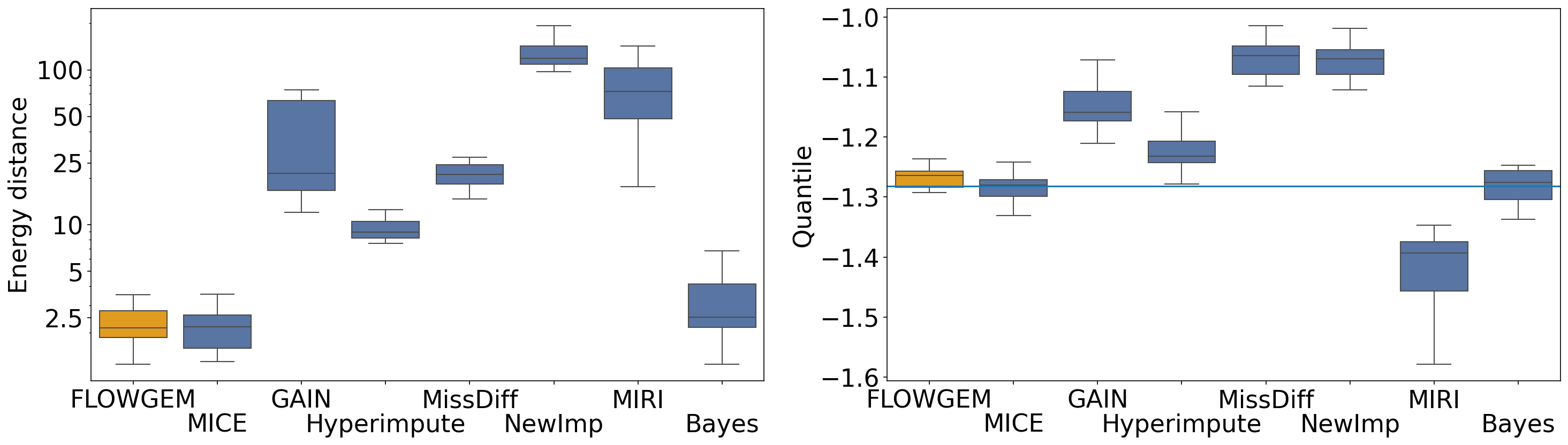}
    \caption{Standardized Energy distance in log-scale (left) and quantile estimate (right) for the simulated example with Gaussian distribution, $n=2000$, and repeated over $B=20$ times. The blue line on the right indicates the true quantile value of the uncontaminated distribution. Our method, \name, belongs to one of the best methods, minimizing the energy distance and estimating the quantile most accurately. We note that both MICE and the Bayes method essentially use Gaussian regression and Gaussian approximations respectively, giving them a somewhat unfair edge in this example. Despite this edge, it turns out that increasing $T$ to 1500 further boosts \name and closes the gap between our method and MICE or Bayes, but we did not include this result here.}
    \label{fig:Gaussianboxplot}
\end{figure}

Finally, we present the Gaussian version of the simulated example in Section \ref{sec:toyexmpl}.  For $d=3$, we take $X \sim \Gauss{0}{\Sigma}$, where $\Gauss{\mu}{\Sigma}$ denotes the Gaussian distribution with mean $\mu$ and covariance matrix $\Sigma$. Here, $\Sigma$ is chosen to have diagonal elements of one and a correlation of 0.7 between $X_1, X_2$, with no correlation between $X_1,X_2$ and $X_3$. We then define the following missingness mechanism: For $(m_1,m_2, m_3)=((0,0,0), (0,1,0), (1,0,0))$, let
    \begin{align*}
        &\PMmXx{m_1}=(\Phi(x_1)+\Phi(x_2))/3\,; \nonumber \\
        &\PMmXx{m_2}=(2-\Phi(x_1))/3\,; \nonumber \\
        &\PMmXx{m_3}=(1-\Phi(x_2))/3\,,
    \end{align*}
    where $\Phi$ is the standard Gaussian distribution function. Figure \ref{fig:Gaussianboxplot} presents the result. Again, we are far better than GAIN, Hyperimpute, MissDiff, NewImp and MIRI, only narrowly beaten by MICE and Bayes. However, we note that both MICE and the Bayes method essentially use Gaussian regression and Gaussian approximations respectively, giving them a somewhat unfair edge in this example. Despite this edge, it turns out that increasing $T$ to 1500 further boosts \name and closes the gap between our method and MICE or Bayes, but we did not include this result here. 
    Figures~\ref{fig:UNIimputations} and \ref{fig:Gaussianimputations} show scatter plots of the first two dimensions of the generated samples for the uniform and Gaussian simulation, respectively, visually confirming the results in Figures~\ref{fig:UNIboxplot} and \ref{fig:Gaussianboxplot}.

\begin{figure}[ht]
    \centering
    \includegraphics[width=0.9\linewidth]{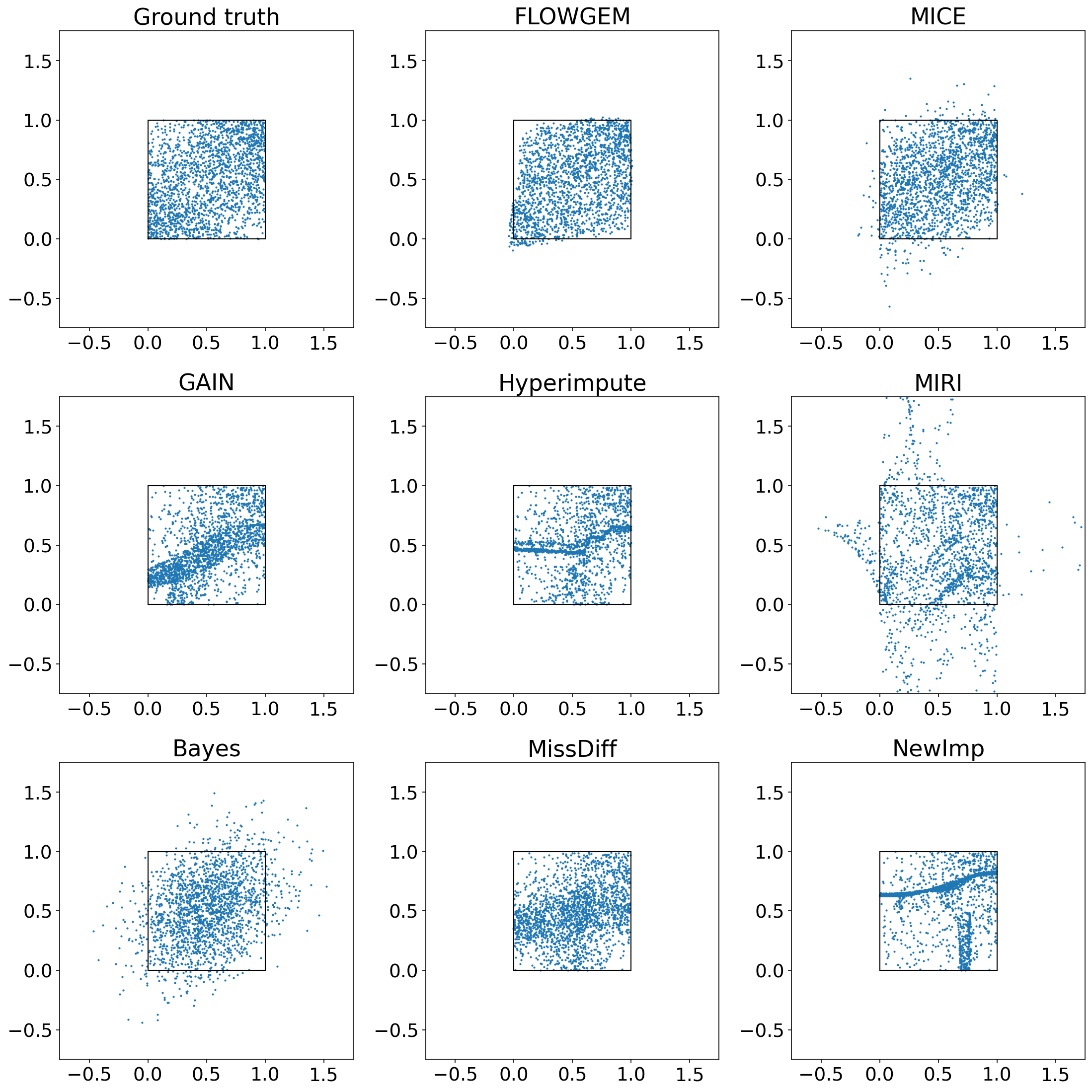}
    \caption{Scatter plots of the first two dimensions of the generated samples for each method, for a single replication of the simulation study in Section~\ref{sec:toyexmpl} with $n=2000$, $d=3$, and uniform distribution. The black square indicates the support $[0,1]^2$ of the true distribution.}
    \label{fig:UNIimputations}
\end{figure}

\begin{figure}[ht]
    \centering
    \includegraphics[width=0.9\linewidth]{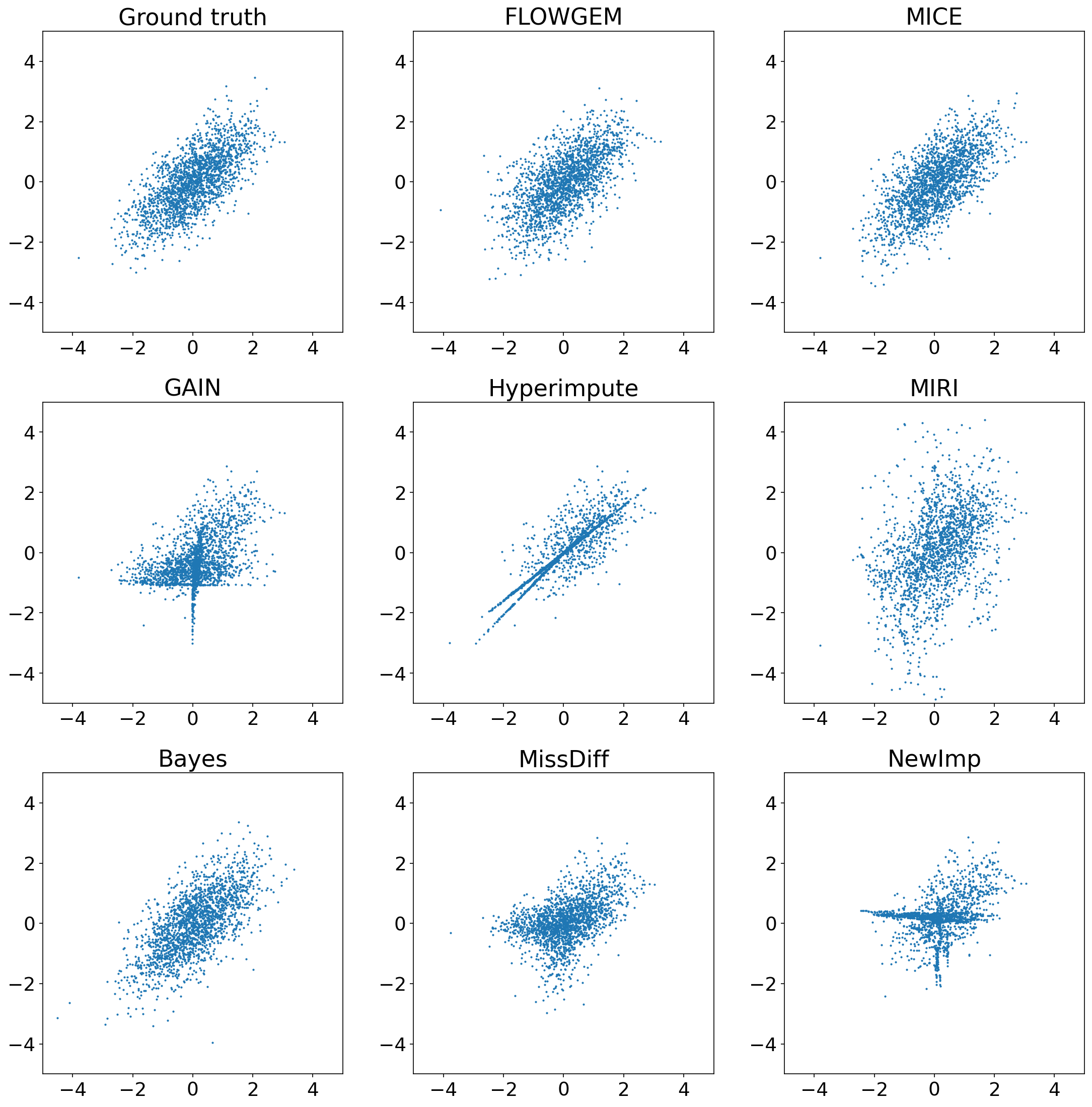}
    \caption{Scatter plots of the first two dimensions of the generated samples for each method, for a single replication of the simulation study in Section~\ref{sec:toyexmpl} with $n=2000$, $d=3$, and Gaussian distribution. We note that both MICE and the Bayes method essentially use Gaussian regression and Gaussian approximations respectively, giving them a somewhat unfair edge in this example.}
    \label{fig:Gaussianimputations}
\end{figure}

\subsection{Proofs}\label{sec:proofs}
Next, we provide the proofs of the results in the main text.


\begin{proof}[Proof of Proposition \ref{prop:KLmin}]
Recall that $\P_{X}^{(M)}$ has density $\pi^{(M)}$, while $\P_{X \mid M}^{(M)}$ has density $\pi_{M}^{(M)}$.
For an arbitrary distribution $P_X$ with density $p$, we may write
\begin{align*}
    &\expec[M \sim \P_M]{\KL{\P_{X|M}^{(M)} \,|\!|\, P_X^{(M)}}}\\
    &\quad= \sum_{m \in \mathcal{M}} \P(M=m) \int \log\left( \frac{\pi_m^{(m)}(x^{(m)})}{p^{(m)}(x^{(m)})}  \right)\pi_m^{(m)}(x^{(m)}) dx^{(m)}\\
    &\quad= \sum_{m \in \mathcal{M}} \P(M=m) \int \log\left( \frac{\pi^{(m)}(x^{(m)})}{p^{(m)}(x^{(m)})} \frac{\P(M=m \mid x^{(m)})}{\P(M=m)}  \right)\pi_m^{(m)}(x^{(m)}) dx^{(m)}\\
    &\quad= \sum_{m \in \mathcal{M}}  \P(M=m)\int \log\left( \frac{\pi^{(m)}(x^{(m)})}{p^{(m)}(x^{(m)})}  \right)\pi_m^{(m)}(x^{(m)}) dx^{(m)}  \\
    &\qquad\quad + \sum_{m \in \mathcal{M}} \P(M=m) \int \log\left(  \frac{\P(M=m \mid x^{(m)})}{\P(M=m)}  \right)\pi_m^{(m)}(x^{(m)}) dx^{(m)}\,.
\end{align*}
The second term above is not influenced by the choice of $P_X$, so we might ignore it for the $\argmin$ and focus on the first term:
\begin{align*}
     &\sum_{m \in \mathcal{M}} \P(M=m) \int \log\left( \frac{\pi^{(m)}(x^{(m)})}{p^{(m)}(x^{(m)})}  \right)\pi_m^{(m)}(x^{(m)}) dx^{(m)}\\
     &\quad=  \sum_{m \in \mathcal{M}} \int \log\left( \frac{\pi^{(m)}(x^{(m)})}{p^{(m)}(x^{(m)})}  \right)\pi^{(m)}(x^{(m)}) \P(M=m \mid x^{(m)}) dx^{(m)}\,.
\end{align*}

Following the same proof as in \cite[Proposition 3.1]{Bayes}, we obtain that under MAR,
\begin{equation*}
   \sum_{m \in \mathcal{M}} \int \log\left( \frac{\pi^{(m)}(x^{(m)})}{p^{(m)}(x^{(m)})}  \right)\pi^{(m)}(x^{(m)}) \P(M=m \mid x^{(m)}) dx^{(m)} \geq 0\,,
\end{equation*}
while zero is clearly attained for $p^{(m)}=\pi^{(m)}$. This shows that $\P_X$ is one solution to \eqref{eq:ML_limit_nonparametric}. Moreover, as shown in \cite[Proposition 3.2]{Bayes}, under MAR, $\P(M=0 \mid x) > 0$ for almost all $x$ implies that,
\begin{equation*}
   \sum_{m \in \mathcal{M}} \int \log\left( \frac{\pi^{(m)}(x^{(m)})}{p^{(m)}(x^{(m)})}  \right)\pi^{(m)}(x^{(m)}) \P(M=m \mid x^{(m)}) dx^{(m)} = 0 \iff \P_X=P_X\,,
\end{equation*}
showing that $\P_X$ is the \emph{unique} solution to \eqref{eq:ML_limit_nonparametric}. On the other hand, if $\P(M=0 \mid x) = 0$ for $x \in A$ with $\P_X(A) > 0$, we can construct a density $P_X$ that is equal to $\P_X$ on $A^c$ and has the same $k \leq d-1$ dimensional marginals, but such that $P_X \neq \P_X$ on $A$. This difference in turn will be masked by $\P(M=0 \mid x) = 0$. 
\end{proof}

\begin{proof}[Proof of Lemma \ref{lem:deriv}]
    Using the chain rule and the law of total probability, we have 
    \begin{align*}
        \frac{\de}{\de \eps}\rvert_{\eps = 0} \mathcal{F}[\rho_\eps] &= \sum_m \P(M=m) \int \frac{\de}{\de \eps}\rvert_{\eps = 0} f\qty(\frac{\pX_m^{(m)}(x^{(m)})}{\rho_\eps^{(m)}(x^{(m)})}) \rho_\eps^{(m)}(x^{(m)}) \de x^{(m)}\\
        &= \sum_m \P(M=m)\int \qty(f\qty(\frac{\pX_m^{(m)}(x^{(m)})}{\rho^{(m)}(x^{(m)})}) - f'\qty(\frac{\pX_m^{(m)}(x^{(m)})}{\rho^{(m)}(x^{(m)})}) \frac{\pX_m^{(m)}(x^{(m)})}{\rho^{(m)}(x^{(m)})}) \\
        &\qquad\qquad\qquad\qquad\quad \cdot  \frac{\de}{\de \eps}\rvert_{\eps = 0} \rho_\eps^{(m)}(x^{(m)}) \de x^{(m)}\\
        &= - \sum_m \P(M=m) \int \qty(h \circ r^{(m)})(x^{(m)}) \frac{\de}{\de \eps}\rvert_{\eps = 0} \rho_\eps^{(m)}(x^{(m)}) \de x^{(m)}\\
        &= - \sum_m \P(M=m) \int \qty(h \circ r^{(m)})(x^{(m)}) \frac{\de}{\de \eps}\rvert_{\eps = 0} \int \rho_\eps(x) \de x^{(1-m)} \de x^{(m)}\\
        &= - \int \sum_m \P(M=m) \qty(h  \circ r^{(m)})(x^{(m)}) \frac{\de}{\de \eps}\rvert_{\eps = 0} \rho_\eps(x) \de x\,.
    \end{align*}
    The result follows from the definition of $\frac{\delta \mathcal{F}[\rho]}{\delta \rho}(x)$ being the first variation of $\mathcal{F}[\rho]$.
\end{proof}

\begin{proof}[Proof of Lemma \ref{lem:key_obs_m}]
    Assume $\psi$ is convex and lower semi-continuous. Then, by duality, we have
    \begin{equation*}
        \psi(u) = \sup_{v \in \text{dom}(\psi^*)} \{vu - \psi^*(v)\}\,.
    \end{equation*}
    It follows that, for each $m \in \mathcal{M}$,
    \begin{align*}
        &D_\psi\qty(\pX_m^{(m)}, \pXnew^{(m)})\\
        &\quad= \int \pXnew^{(m)}(x^{(m)}) \psi\qty(\frac{\pX_m^{(m)}(x^{(m)})}{\pXnew^{(m)}(x^{(m)})}) \de x^{(m)}\\
        &\quad= \int \pXnew^{(m)}(x^{(m)}) \sup_{v \in \text{dom}(\psi^*)} \left\{v \cdot \frac{\pX_m^{(m)}(x^{(m)})}{\pXnew^{(m)}(x^{(m)})} - \psi^*(v) \right\} \de x^{(m)}\\
        &\quad= \sup_{\dfunc_m:\, \R^{d_m} \to \R} \left\{ \int \pXnew^{(m)}(x^{(m)}) \qty(\dfunc_m(x^{(m)}) \frac{\pX_m^{(m)}(x^{(m)})}{\pXnew^{(m)}(x^{(m)})} - \psi^*\qty(\dfunc_m(x^{(m)}))) \de x^{(m)}\right\}\\
        &\quad= \sup_{\dfunc_m:\, \R^{d_m} \to \R} \left\{ \int \dfunc_m(x^{(m)}) \pX_m^{(m)}(x^{(m)}) \de x^{(m)} - \int \psi^*\qty(\dfunc_m(x^{(m)})) \pXnew^{(m)}(x^{(m)}) \de x^{(m)}\right\}\,.
    \end{align*}
    Since, by duality, $\psi^*(\psi'(r)) = r \psi'(r) - \psi(r)$, the specific choice $\dfunc_m = \psi' \circ r^{(m)}$ yields
    \begin{align*}
        &\int \dfunc_m(x^{(m)}) \pX_m^{(m)}(x^{(m)}) \de x^{(m)} - \int \psi^*(\dfunc_m(x^{(m)})) \pXnew^{(m)}(x^{(m)}) \de x^{(m)}\\
        &\quad= \int \psi'(r^{(m)}(x^{(m)})) \pX_m^{(m)}(x^{(m)}) \de x^{(m)} - \int \psi^*(\psi'(r^{(m)}(x^{(m)}))) \pXnew^{(m)}(x^{(m)}) \de x^{(m)}\\
        &\quad= \int \psi'(r^{(m)}(x^{(m)})) \pX_m^{(m)}(x^{(m)}) \de x^{(m)}\\
        &\qquad\quad- \int \qty(r^{(m)}(x^{(m)}) \psi'(r^{(m)}(x^{(m)})) - \psi(r^{(m)}(x^{(m)}))) \pXnew^{(m)}(x^{(m)}) \de x^{(m)}\\
        &\quad= \int \psi(r^{(m)}(x^{(m)})) \pXnew^{(m)}(x^{(m)}) \de x^{(m)}\\
        &\quad= D_\psi\qty(\pX_m^{(m)}, \pXnew^{(m)})\,.
    \end{align*}
    The result follows from the observation that $h(r) = rf'(r)-f(r) = \psi'(r)$.

    Note that \citet{liu2024minimizing} state that $\psi$ is chosen such that $\psi'(r)$ is equal to $r f'(r) - f(r)$ \emph{up to some constant}. Accordingly, they use $\psi(r) = (r-1)^2/2$ and $\psi^*(\dfunc) = \dfunc^2/2 + \dfunc$ for the KL divergence. We want to emphasize that this is not a valid choice when using $f(r)=r \log(r)$, since Equation~\eqref{eq:key_obs} only holds if $h(r) = \psi'(r)$. This is a minor issue, though, as it only influences the intercept $b_t$ in equation \eqref{eq:optim_wb} and not the sought-after slope $w_t$.
\end{proof}

\begin{proof}[Proof of Proposition \ref{prop:approx_error}]
    For fixed time step $t \in \{0, 1, \dots, T-1\}$, missingness pattern $m \in \mathcal{M}$ and $x^* \in \mathcal{X}$, \citet[Theorem 4.8]{liu2024minimizing} states that
    \begin{equation}    
    \label{eq:err_bound_liu}
         \left|\!\left| w_{t,m}(x^*) - \nabla (h \circ r_t^{(m)})(x^*)\right|\!\right|_2 \le C^{(1)}_{t, m, x^*} \frac{1}{\sqrt{n_m \sigma^{d_m}}} + C^{(2)}_{t, m, x^*}\sigma^2\,.
    \end{equation}
    By the triangle inequality, we have
    \begin{align*}
        &\left|\!\left| \sum_{m \in \mathcal{M}} \frac{n_m}{n} w_{t,m}(x^*) - \expec[M]{\nabla (h \circ r_t^{(M)})(x^*)}\right|\!\right|_2 \\
        &\quad\le \sum_{m \in \mathcal{M}} \frac{n_m}{n} \left|\!\left| w_{t,m}(x^*) - \nabla (h \circ r_t^{(m)})(x^*)\right|\!\right|_2\\
        &\qquad+ \left|\!\left| \sum_{m \in \mathcal{M}} \frac{n_m}{n} \nabla (h \circ r_t^{(m)})(x^*) - \expec[M]{\nabla (h \circ r_t^{(M)})(x^*)}\right|\!\right|_2\,.
    \end{align*}
    The second term is the standard Monte Carlo estimation error vanishing at a rate of $\bigo{n^{-1/2}}$, while the first term can be bounded by
    ~\\
    \begin{align*}
        \sum_{m \in \mathcal{M}} \frac{n_m}{n} \left|\!\left| w_{t,m}(x^*) - \nabla (h \circ r_t^{(m)})(x^*)\right|\!\right|_2 &\le \sum_{m \in \mathcal{M}} \frac{n_m}{n} \qty(C^{(1)}_{t, m, x^*} \frac{1}{\sqrt{n_m \sigma^{d_m}}} + C^{(2)}_{t, m, x^*}\sigma^2)\\
        &\le \qty(\max_{m \in \mathcal{M}} C^{(1)}_{t, m, x^*}) \frac{1}{n \sqrt{\sigma^{d}}} \sum_{m \in \mathcal{M}} \sqrt{n_m} + \max_{m \in \mathcal{M}} C^{(2)}_{t, m, x^*}\sigma^2\\
        &\le \qty(\max_{m \in \mathcal{M}} C^{(1)}_{t, m, x^*}) \frac{1}{n \sqrt{\sigma^{d}}} \sqrt{\abs{\mathcal{M}} \sum_{m \in \mathcal{M}} n_m} + \max_{m \in \mathcal{M}} C^{(2)}_{t, m, x^*}\sigma^2\\
        &= \qty(\max_{m \in \mathcal{M}} C^{(1)}_{t, m, x^*}) \sqrt{\frac{\abs{\mathcal{M}}}{n \sigma^{d}}} + \max_{m \in \mathcal{M}} C^{(2)}_{t, m, x^*}\sigma^2\,,
    \end{align*}
    where the first inequality applies \eqref{eq:err_bound_liu}, the second one follows from $\sigma \le 1$ and $d_m \le d$, the third one uses the Cauchy-Schwartz inequality, and the last equality holds because $\sum_m n_m = n$.
\end{proof}